%% file: main.tex
\input{_0_vars}

\ifdefined\isarxiv
\documentclass[11pt]{article}
\usepackage[numbers]{natbib}
\else
\documentclass{article} 
\usepackage{iclr2026_conference,times}

\fi

\ifdefined\isarxiv

\usepackage{amsmath}
\usepackage{amsthm}
\usepackage{amssymb}
\usepackage{algorithm}
\usepackage{subfig}
\usepackage{algpseudocode}
\usepackage{graphicx}
\usepackage{grffile}
\usepackage{wrapfig,epsfig}
\usepackage{url}
\usepackage{xcolor}
\usepackage{epstopdf}

\usepackage{bbm}
\usepackage{dsfont}

\else 

\usepackage{hyperref}
\usepackage{url}

\fi

\ifdefined\isarxiv

\else

\usepackage{amsmath}
\usepackage{amsthm}
\usepackage{amssymb}
\usepackage{algorithm}
\usepackage{subfig}
\usepackage{algpseudocode}
\usepackage{graphicx}
\usepackage{grffile}
\usepackage{wrapfig,epsfig}
\usepackage{url}
\usepackage{xcolor}
\usepackage{epstopdf}

\usepackage{bbm}
\usepackage{dsfont}

\fi
 
\allowdisplaybreaks

\ifdefined\isarxiv

\usepackage{tikz}
\usepackage{hyperref}  
\hypersetup{colorlinks=true,citecolor=blue,linkcolor=blue} 
\usetikzlibrary{arrows}
\usepackage[margin=1in]{geometry}

\else
\fi
 
\graphicspath{{./figs/}}

\theoremstyle{plain}
\newtheorem{theorem}{Theorem}[section]
\newtheorem{lemma}[theorem]{Lemma}
\newtheorem{definition}[theorem]{Definition}

\newtheorem{fact}[theorem]{Fact}
\newtheorem{remark}[theorem]{Remark}

\ifdefined\isarxiv

\else
\renewcommand\cite\citep
\fi

\input{_1_maths}

\begin{document}

\ifdefined\isarxiv

\date{}
\title{\paperTitle}
\author{\paperAuthor}

\else

\title{\paperTitle}

\author{Antiquus S.~Hippocampus, Natalia Cerebro \& Amelie P. Amygdale \thanks{ Use footnote for providing further information
about author (webpage, alternative address)---\emph{not} for acknowledging
funding agencies.  Funding acknowledgements go at the end of the paper.} \\
Department of Computer Science\\
Cranberry-Lemon University\\
Pittsburgh, PA 15213, USA \\
\texttt{\{hippo,brain,jen\}@cs.cranberry-lemon.edu} \\
\And
Ji Q. Ren \& Yevgeny LeNet \\
Department of Computational Neuroscience \\
University of the Witwatersrand \\
Joburg, South Africa \\
\texttt{\{robot,net\}@wits.ac.za} \\
\AND
Coauthor \\
Affiliation \\
Address \\
\texttt{email}
}

%

\newcommand{\fix}{\marginpar{FIX}}
\newcommand{\new}{\marginpar{NEW}}

\maketitle

\fi

\ifdefined\isarxiv
  \maketitle
  \begin{abstract}

\input{00_abstract}

  \end{abstract}


\else

\begin{abstract}
\input{00_abstract}
\end{abstract}

\fi


\input{_2_body}

\ifdefined\isarxiv
\bibliographystyle{alpha}
\bibliography{ref}
\else
\fi

\end{document}

%% file: _0_vars.tex
\def\isarxiv{1}

\def\paperTitle{Tracking High-order Evolutions via Cascading Low-rank Fitting}

\def\paperAuthor{
Zhao Song\thanks{\texttt{magic.linuxkde@gmail.com}. Work done while at Simons Institute for Theory of Computing, UC Berkeley.}}


%% file: _1_maths.tex
\newcommand{\R}{\mathbb{R}}

\renewcommand{\d}{\mathrm{d}}

\newcommand{\Tmat}{{\cal T}_{\mathrm{mat}}}

\DeclareMathOperator{\rank}{rank}
\DeclareMathOperator{\diag}{diag}

%% file: 00_abstract.tex
Diffusion models have become the de facto standard for modern visual generation, including well-established frameworks such as latent diffusion and flow matching. Recently, modeling high-order dynamics has emerged as a promising frontier in generative modeling. Rather than only learning the first-order velocity field that transports random noise to a target data distribution, these approaches simultaneously learn higher-order derivatives, such as acceleration and jerk, yielding a diverse family of higher-order diffusion variants. 
To represent higher-order derivatives, naive approaches instantiate separate neural networks for each order, which scales the parameter space linearly with the derivative order. To overcome this computational bottleneck, we introduce cascading low-rank fitting, an ordinary differential equation inspired method that approximates successive derivatives by applying a shared base function augmented with sequentially accumulated low-rank components. 

Theoretically, we analyze the rank dynamics of these successive matrix differences. We prove that if the initial difference is linearly decomposable, the generic ranks of high-order derivatives are guaranteed to be monotonically non-increasing. Conversely, we demonstrate that without this structural assumption, the General Leibniz Rule allows ranks to strictly increase. Furthermore, we establish that under specific conditions, the sequence of derivative ranks can be designed to form any arbitrary permutation. Finally, we present a straightforward algorithm to efficiently compute the proposed cascading low-rank fitting.

%% file: _2_body.tex

\input{01_intro}

\input{02_preli}

%% file: 01_intro.tex
\section{Introduction}

Generative frameworks such as latent diffusion~\cite{rbl+22,brl+23,stable_video-diffusion} and flow matching~\cite{flow_matching,rectified_flow,av22} have become foundational to modern video generation, enabling state-of-the-art video generation systems like Sora~\cite{sora}, Seedance~\cite{seedance_10,seedance_15}, and Stable Video Diffusion~\cite{brl+23,stable_video-diffusion}. Mathematically, if we let $g(t)$ denote the trajectory of a sample as it evolves from random noise to the target data distribution, several diffusion models like latent diffusion aim to learn a network $f_0$ that directly approximates $g$.  Flow matching focuses on learning a vector field $f_1$ that approximates the first-order time derivative, $\frac{\d g}{\d t}$. Given the success of these methods and the importance of noise scheduling in diffusion models~\cite{kaal22}, it is a natural progression to model higher-order~\cite{shl+25} trajectory dynamics. To learn the derivatives of $g$ up to order $k$, the naive approach is to construct $k+1$ separate neural networks ($f_0, f_1, \dots, f_k$). However, this causes an unacceptable parameter blowup, since the total parameter space expands to $(k+1)N$ if a single network requires $N$ parameters. A standard strategy to constrain this growth is to apply Low-Rank Adaptation (LoRA) \cite{hsw+22_lora}. Under a standard LoRA paradigm, one might learn a unified base function $f$ parameterized by dense weights $W$, and append independent low-rank weight matrices $A_i, B_i$ to specifically approximate the $i$-th derivative. In this paper, we demonstrate that by viewing the system through an ODE formulation, we uncover a more powerful structural intuition: \textit{cascading low-rank fitting}. Because higher-order derivatives inherently build upon lower-order states, our cascading method evaluates the $i$-th derivative of $g$ not through isolated adapters, but by applying $f$ over the base weights $W$ synergistically accumulated with all sequential low-rank components up to order $i$, denoted by $\{ A_j, B_j \}_{j \leq i}$. Intuitively, we use $f( ( W - \sum_{j \leq i} A_j B_j) x)$ to interpolate $\frac{d^i g(t)}{\d^i t} = ( C(t) - \sum_{j \leq i} L_j(t) ) x$. While our explicit analysis focuses on linear weight transformations, extending this framework to softmax functions in modern attention mechanisms is straightforward. By leveraging established polynomial approximation techniques \cite{acss20,aa22,as23,as24_neurips,as24_iclr,as25_rope,cpa26}, our method can be directly adapted to optimize the softmax operations over quadratic pairwise correlations found in modern attention mechanisms.

%% file: 02_preli.tex
Intuitively, when applying diffusion processes to video generation, it is highly desirable for the rank $r_i$ to be monotonically non-increasing or to decay rapidly, where $r_i = \rank(L_i(t))$ and $L_i(t) = \frac{\d^{i-1} W(t)}{\d t^{i-1}} - \frac{\d^i W(t)}{\d t^i}$. This decay property allows us to allocate progressively smaller LoRA adapters for higher-order dynamics, improving parameter efficiency. In this work, we systematically explore the scenarios under which this rank sequence naturally shrinks, alongside the conditions that allow for arbitrary rank patterns. To establish this foundation, we first introduce the concept of linear decomposability for matrix functions.
\begin{definition}[linearly decomposable]
For a matrix $U(t)$ that is parameterized with $t$, we say $U(t)$ is $r$-linearly decomposable, if there exists $r$ functions (can depend on $t$) $f_1(t), \cdots, f_r(t)$, $2r$ vectors  (independent of $t$) $u_1, \cdots, u_r$, $v_1, \cdots, v_r$ satisfy that $U(t) = \sum_{j=1}^r f_{j}(t) u_j v_j^\top$.
\end{definition}
Under the assumption of linear decomposability, we prove the following result.
\begin{lemma}
Let $W(t)$ denote an $n \times n$ matrix. Let $k$ denote a positive integer. For each $i \in [k]$, we define $L_i(t) = \frac{\d^{i-1} W(t)}{\d t^{i-1}} - \frac{\d^i W(t)}{\d t^i}$.  Assume that $L_1(t)$ is $r_1$-linearly decomposable, i.e., $L_1(t) = \sum_{j=1}^{r_1} f_j(t) u_j v_j^\top$. For each $i \geq 2,3,\cdots, k$, we define the rank of $L_i(t)$ to be $r_i$. Then we can show that $r_1 \geq r_2 \geq \dots \geq r_k$.
\end{lemma}
\begin{proof}
Let $L_1(t) = W(t) - \frac{\d W(t)}{\d t}$. We are given that $L_1(t)$ has  rank $r_1$ and can be expressed as:
$ L_1(t) = \sum_{j=1}^{r_1} f_j(t) u_j v_j^\top $. 
Because $L_1(t)$ is formed by a sum of exactly $r_1$ rank-1 matrices and achieves a rank of exactly $r_1$, the set of constant column vectors $\{u_1, u_2, \dots, u_{r_1}\}$ must be linearly independent. Similarly, the set of constant vectors $\{v_1, v_2, \dots, v_{r_1}\}$ must also be linearly independent. If either set contained linear dependencies, the rank of their sum would be strictly less than $r_1$. 
We can rewrite $L_1(t)$ compactly in matrix form:
$ L_1(t) = U \Sigma_1(t) V^\top $
where $U = [u_1, \dots, u_{r_1}]$ and $V = [v_1, \dots, v_{r_1}]$ are $n \times r_1$ matrices with full column rank. The matrix $\Sigma_1(t) = \diag(f_1(t), \dots, f_{r_1}(t))$ is an $r_1 \times r_1$ diagonal matrix. Because $U$ and $V^\top$ have full rank, the rank of $L_1(t)$ is purely determined by the rank of $\Sigma_1(t)$, which is exactly the number of generically non-zero functions $f_j(t)$ on its diagonal. For any integer $i \geq 2$, we are evaluating the rank of the matrix difference:
$ L_i(t) = \frac{\d^{i-1} W(t)}{\d t^{i-1}} - \frac{\d^i W(t)}{\d t^i} $. Notice that $L_i(t)$ is simply the $(i-1)$-th derivative of $L_1(t)$ with respect to $t$. Because $u_j$ and $v_j$ are constant vectors, the differentiation passes through the matrices $U$ and $V^\top$ and applies strictly to the scalar functions:
$ L_i(t) = U ( \frac{\d^{i-1}}{\d t^{i-1}} \Sigma_1(t) ) V^\top = U \Sigma_i(t) V^\top $
where $\Sigma_i(t) = \diag(f_1^{(i-1)}(t), \dots, f_{r_1}^{(i-1)}(t))$. Again, because $U$ and $V^\top$ maintain full column and row rank respectively, the rank of $L_i(t)$, defined as $r_i$, is equal to the rank of the diagonal matrix $\Sigma_i(t)$. This rank corresponds to the number of generically non-zero functions in the set $\{f_1^{(i-1)}(t), \dots, f_{r_1}^{(i-1)}(t)\}$. When taking successive derivatives of functions, if a function is identically zero across the domain (i.e., $f_j^{(i-1)}(t) \equiv 0$), its derivative must also be identically zero ($f_j^{(i)}(t) \equiv 0$). It is impossible for the derivative of a zero function to yield a non-zero function.  Therefore, as $i$ increases, the number of identically zero entries on the diagonal of $\Sigma_i(t)$ can only increase or stay the same. Consequently, the number of non-zero entries (which dictates the rank) must monotonically non-increase:
$ \rank(\Sigma_1(t)) \geq \rank(\Sigma_2(t)) \geq \dots \geq \rank(\Sigma_k(t)) $. This directly proves that $r_1 \geq r_2 \geq \dots \geq r_k$.
\end{proof}

\paragraph{Linearly decomposable assumption is necessary}
Next, we will show that the linearly decomposable assumption is necessary for rank shrinking. If we don't have such assumptions, it is easy to construct some counter-examples which have that $r_1 < r_2$. Let us consider a $2 \times 2$ matrix. It is not hard to see that $L_2(t) = \frac{\d L_1(t) }{\d t}$, thus $r_2 = \rank( \frac{\d L_1(t)}{\d t} )$. Let us consider $L_1(t) = [1 ~t ; t~ t^2 ] = [ 1 ~ t]^\top [1~ t]$ which has rank-$1$. Next, we can compute $L_2(t) = [0 ~1; 1~ 2t]$. This matrix has rank-$2$. Next we generalize the idea to high-orders by using $L_1(t) =[1~ t ~ \cdots t^{n-1}]^\top [1~ t ~ \cdots t^{n-1}]$.

\begin{lemma}[rank could strictly increase after derivatives]
For any integer $k$, there exists an $n \times n$ matrix $W(t)$ (with $n \geq k$) such that if we define $L_1(t) = W(t) - \frac{\d W(t)}{\d t}$ and $L_i(t) = \frac{\d^{i-1} W(t)}{\d t^{i-1}} - \frac{\d^i W(t)}{\d t^i}$ for $i \geq 2$, the generic ranks $r_i = \rank(L_i(t))$ satisfy: $ r_k > r_{k-1} > \dots > r_2 > r_1 $.
\end{lemma}

\begin{proof}
Let $n \geq k$. We define $W(t)$ to be the solution to the differential equation $W(t) - W'(t) = L_1(t)$, where $L_1(t) = u(t) u(t)^\top$ and $u(t) = [1, t, t^2, \dots, t^{n-1}]^\top$. Because $L_1(t)$ is formed by the outer product of a single non-zero vector function, its generic rank is $r_1 = 1$. For any $i \geq 2$, we have $L_i(t) = \frac{\d^{i-1}}{\d t^{i-1}} L_1(t)$. By the General Leibniz Rule, the $(i-1)$-th derivative of the outer product $u(t) u(t)^\top$ is:
$ L_i(t) = \sum_{j=0}^{i-1} \binom{i-1}{j} u^{(i-1-j)}(t) {u^{(j)}(t)}^\top $. 
We can express this summation compactly in matrix form as: $ L_i(t) = U_i(t) D_i V_i(t)^\top $.
Here $U_i(t) = [ u^{(i-1)}(t) ~ u^{(i-2)}(t) ~ \dots ~ u^{(0)}(t) ]$ is an $n \times i$ matrix.  $V_i(t) = [ u^{(0)}(t) ~ u^{(1)}(t) ~ \dots ~ u^{(i-1)}(t) ]$ is an $n \times i$ matrix. $D_i = \diag( \binom{i-1}{0}, \binom{i-1}{1}, \dots, \binom{i-1}{i-1} )$ is an $i \times i$ diagonal matrix. 
To determine the generic rank of $L_i(t)$, it is sufficient to evaluate its rank at $t = 0$. The $m$-th derivative of $u(t)$ evaluated at $t=0$ isolates the $(m+1)$-th component of the vector. Specifically, $u^{(m)}(0) = m! e_{m+1}$, where $e_{m+1}$ is the standard basis vector in $\mathbb{R}^n$. 
Evaluating our component matrices at $t=0$, we obtain:
$ U_i(0) = [ (i-1)! e_i ~ (i-2)! e_{i-1} ~ \dots ~ 0! e_1 ] $ and
$ V_i(0) = [ 0! e_1 ~ 1! e_2 ~ \dots ~ (i-1)! e_i ] $. Because $i \leq k \leq n$, the columns of $U_i(0)$ are non-zero scalar multiples of distinct standard basis vectors $\{e_1, e_2, \dots, e_i\}$. Therefore, $U_i(0)$ has full column rank $i$. By the exact same logic, $V_i(0)^\top$ has full row rank $i$. Furthermore, because all binomial coefficients are strictly positive, the diagonal matrix $D_i$ is non-singular and has rank $i$. Because $L_i(0)$ is the product of an $n \times i$ matrix of rank $i$, an $i \times i$ matrix of rank $i$, and an $i \times n$ matrix of rank $i$, it follows that:
$ \rank(L_i(0)) = i $.  
Since the rank of a matrix function evaluated at a specific point serves as a lower bound for its generic rank over the function space, and $L_i(t)$ is factored through an $i \times i$ inner dimension, the generic rank of $L_i(t)$ is exactly $r_i = i$.  
Evaluating this for $i = 1, 2, \dots, k$, we obtain $r_i = i$, which directly yields:
$ r_k > r_{k-1} > \dots > r_2 > r_1 $
completing the proof.
\end{proof}

\paragraph{The matrices \texorpdfstring{$L_1, L_2, \cdots, L_k$}{} can be constructed via solving ODE}

Let the target low-rank matrix be $L_1(t)$. Your condition is: $ W(t) - \frac{\d W(t)}{\d t} = L_1(t)$. Rearranging this gives a standard first-order linear ordinary differential equation (ODE) for the matrix $W(t)$: $\frac{\d W(t)}{\d t} - W(t) = -L_1(t)$. Using the integrating factor $e^{-t}$, the general solution to this equation is: $ W(t) = e^t C - e^t \int e^{-t} L_1(t) \d t $
where $C$ is an arbitrary $n \times n$ constant matrix (which you can choose to be full-rank so that $W(t)$ is not trivially a low-rank matrix itself).

{Constant low-rank difference.}
If we just need the difference to be a constant low-rank matrix $L_1$, the integration becomes trivial. Let $L_1 = \sum_{i=1}^r u_i v_i^{\top}$ where $u_i$ and $v_i$ are column vectors, and $r \ll n$ is your desired low rank. 
Plugging a constant $L$ into our integral gives: $W(t) = e^t C - e^t (-e^{-t} L_1) = e^t C + L_1$. Let us verify the solutions: 1) $W(t) = e^t C + L_1$; 2) $\frac{\d W(t)}{\d t} = e^t C$; 3) $W(t) - \frac{\d W(t)}{\d t} = (e^t C + L_1) - e^t C = L_1$. Because $L_1$ is constructed as a rank-$r$ matrix, the condition is perfectly satisfied.

{Time-varying low-rank difference.} 
If we want the low-rank matrix to change with time, we can choose an $L_1(t)$ that is easy to integrate. Let's design $W(t)$ such that the difference is a rank-1 matrix that scales with $t$. Let $u$ and $v$ be constant $n \times 1$ vectors. We define our matrix as: $ W(t) = e^t C + t u v^{\top} $. Let us verify the solutions: 1) $W(t) = e^t C + t u v^{\top}$; 2) The derivative with respect to $t$ is: $\frac{\d W(t)}{\d t} = e^t C + u v^{\top}$; 3) Subtracting the two: $ W(t) - \frac{\d W(t)}{\d t} = (e^t C + t u v^{\top}) - (e^t C + u v^{\top}) = (t - 1) u v^{\top} $.

The above idea can be easily generalized to higher orders. Similarly, using the ODE formulation, we know that $W(t)$ should include a component $e^t C$ where $C$ is a constant $n \times n$ matrix. Then, we let $W(t) = e^t C + t^k u v^\top$. Let us verify the solution $\frac{\d^i W(t)}{\d t^i} = e^t C + \frac{k!}{(k-i)!} t^{k-i} u v^\top$ for all $i =1, \dots, k$. Further we know that, $L_i(t)= \frac{\d^{i-1} W(t)}{\d t^{i-1} }- \frac{\d^{i} W(t)}{\d t^{i} } = \frac{k!}{(k-i+1)!} t^{k-i} (t - k + i - 1) u v^\top$ for all $i =1, \dots, k$.

\subsection{Rank pattern of different derivatives}

Let $q : [k] \rightarrow [n]$ be an arbitrary sequence of target ranks, and let $L_{i}(t) = \frac{\d L_{i-1}(t)}{\d t}$ for all $i \in \{2, \dots, k\}$. A natural question arises: is it possible to construct an initial matrix $L_1(t)$ such that its successive derivatives achieve exactly these generic ranks, i.e., $\rank(L_i(t)) = q(i)$ for all $i \in [k]$? This is generally impossible because the generic rank of a matrix derivative is strictly bounded by twice the generic rank of the preceding matrix ($\rank(L_i(t)) \leq 2 \rank(L_{i-1}(t))$).
\begin{fact}[growth of rank]
Let $L_{i-1}(t)$ be an $n \times n$ matrix function and let $L_i(t) = \frac{\d L_{i-1}(t)}{\d t}$. The generic rank of $L_i(t)$ is bounded by: $ 0 \leq \rank(L_i(t)) \leq 2 \rank(L_{i-1}(t)) $.
\end{fact}

\begin{proof}
We prove the lower bound and upper bound separately.

 {Lower Bound:} 
The lower bound is trivially $0$ and is completely decoupled from the rank of $L_{i-1}(t)$. Consider the case where $L_{i-1}(t) = C$, where $C$ is a constant $n \times n$ matrix of arbitrary rank $r \leq n$. Because $C$ is constant, its derivative with respect to $t$ is the zero matrix: $ L_i(t) = \frac{\d C}{\d t} = \mathbf{0}_{n \times n}$. 
The rank of the zero matrix is exactly $0$. Therefore, the rank can drop from any arbitrary $r$ to $0$ in a single derivative step, establishing $0$ as the tightest possible universal lower bound.

{Upper Bound:}
Let $r = \rank(L_{i-1}(t))$. Over the field of functions, any $n \times n$ matrix $L_{i-1}(t)$ of rank $r$ admits a rank factorization: $ L_{i-1}(t) = U(t) V(t)^\top$
where $U(t)$ and $V(t)$ are $n \times r$ matrices. Applying the matrix product rule to compute the first derivative yields:$ L_i(t) = \frac{\d L_{i-1}(t)}{\d t} = U'(t) V(t)^\top + U(t) {V'(t)}^\top $. 
The column space of $L_i(t)$ is spanned entirely by linear combinations of the columns of $U'(t)$ and the columns of $U(t)$. Because both $U'(t)$ and $U(t)$ have exactly $r$ columns, the dimension of their combined column space cannot exceed $2r$. Therefore, the maximum possible generic rank of the derivative is bounded by: $\rank(L_i(t)) \leq 2r = 2 \rank(L_{i-1}(t))$. 
This completes the proof of both bounds.
\end{proof}

While the preceding claim does not hold for generic rank over the function space, it becomes achievable if we evaluate the matrix ranks at a specific stationary point, such as $t=0$. We present this simplified result below.
\begin{lemma}[positive result, rank matching at $t=0$] 
Let $q : [k] \rightarrow [n]$. Let $L_{i}(t) = \frac{\d L_{i-1}(t)}{\d t}$ for all $i=2, \cdots, k$. It is possible to design an $n \times n$ matrix function $L_1(t)$ such that for all $i\in [k]$, the rank evaluated at $t=0$ satisfies $\rank(L_i(0)) = q(i)$.
\end{lemma}

\begin{proof}
Let $C_1, C_2, \dots, C_k$ be a sequence of constant $n \times n$ matrices chosen strictly such that $\rank(C_i) = q(i)$ for each $i \in [k]$. Such matrices always exist since $q(i) \le n$. We construct $L_1(t)$ as a matrix polynomial: $ L_1(t) = \sum_{j=1}^k \frac{t^{j-1}}{(j-1)!} C_j $. 
For any $i \in [k]$, the $(i-1)$-th derivative of $L_1(t)$ is given by:
$ L_i(t) = \frac{\d^{i-1} L_1(t)}{\d t^{i-1}} = \sum_{j=i}^k \frac{t^{j-i}}{(j-i)!} C_j $. 
Evaluating this derivative matrix at $t=0$, all terms containing $t$ vanish completely. The only remaining term is the first term of the sum where $j=i$: $ L_i(0) = C_i$. 
Consequently, $\rank(L_i(0)) = \rank(C_i) = q(i)$. This holds true for all $i \in [k]$, perfectly satisfying any arbitrary function $q$ and completing the proof.
\end{proof}

\begin{lemma}[negative result] 
Let $k \ge 3$ and $n \ge 4$. There exists a sequence $q : [k] \rightarrow [n]$ satisfying the step-wise bound $q(i) \leq 2q(i-1)$ for all $i \in \{2, \dots, k\}$, such that for the successive derivatives $L_{i}(t) = \frac{\d L_{i-1}(t)}{\d t}$, no $n \times n$ matrix function $L_1(t)$ can achieve the generic ranks $\rank(L_i(t)) = q(i)$ for all $i \in [k]$.
\end{lemma}

\begin{proof}
We prove this lemma by explicitly constructing a target sequence $q$ that satisfies the step-wise condition $q(i) \leq 2q(i-1)$, but cannot be achieved by any matrix function $L_1(t)$ due to the global bounds imposed by the General Leibniz Rule.

Let $k=3$ and $n \ge 4$. Consider the target rank sequence defined by $q(1) = 1$, $q(2) = 2$, and $q(3) = 4$. This sequence strictly satisfies the step-wise premise, as $q(2) \leq 2q(1)$ and $q(3) \leq 2q(2)$.

Assume for the sake of contradiction that there exists an $n \times n$ matrix function $L_1(t)$ that achieves these generic ranks. Because $\rank(L_1(t)) = q(1) = 1$, we can factorize it over the field of functions as the outer product of two vectors: $L_1(t) = u(t) v(t)^\top $
where $u(t)$ and $v(t)$ are $n \times 1$ vector functions.

The matrix $L_3(t)$ is defined as the second derivative of $L_1(t)$. Applying the General Leibniz Rule to the outer product, we obtain:
$ L_3(t) = \frac{\d^2 L_1(t)}{\d t^2} = u''(t) v(t)^\top + 2 u'(t) v'(t)^\top + u(t) v''(t)^\top $.

The resulting matrix $L_3(t)$ is expressed entirely as the sum of exactly three rank-1 matrices. By the subadditivity of matrix rank, the generic rank of a sum is bounded by the sum of the ranks. Therefore, the maximum possible generic rank of $L_3(t)$ is: $ \rank(L_3(t)) \leq \rank(u''v^\top) + \rank(u'v'^\top) + \rank(uv''^\top) \leq 1 + 1 + 1 = 3 $.

However, the target sequence requires $\rank(L_3(t)) = q(3) = 4$. Since $4 \leq 3$ is an impossibility, no such matrix function $L_1(t)$ can exist. This contradiction proves that the step-wise bound is insufficient to guarantee the existence of $L_1(t)$, thus completing the proof.
\end{proof}

\begin{fact}[high-order growth of rank]
Let $L_{i-1}(t)$ be an $n \times n$ matrix function and let $L_i(t) = \frac{\d L_{i-1}(t)}{\d t}$. The generic rank of $L_i(t)$ is bounded by:
$ 0 \leq \rank(L_i(t)) \leq i \cdot \rank(L_{1}(t))$.
\end{fact}

\begin{proof}
We prove the lower bound and upper bound separately.

 {Lower Bound:} 
The lower bound is trivially $0$. If $L_1(t) = C$, where $C$ is a constant $n \times n$ matrix of arbitrary rank $r$, its first derivative $L_2(t)$ is the zero matrix. Consequently, all higher-order derivatives $L_i(t)$ for $i \ge 2$ will also be the zero matrix, which has a generic rank of exactly $0$. 

 {Upper Bound:}
Let $r = \rank(L_1(t))$. Over the field of functions, any $n \times n$ matrix function $L_1(t)$ of generic rank $r$ admits a global rank factorization:
$ L_1(t) = U(t) V(t)^\top $
where $U(t)$ and $V(t)$ are $n \times r$ matrix functions. 
By the definition provided, $L_i(t)$ represents the $(i-1)$-th derivative of $L_1(t)$ with respect to $t$. Applying the General Leibniz Rule to the matrix product $U(t) V(t)^\top$, we can express $L_i(t)$ as:
\[ L_i(t) = \frac{\d^{i-1}}{\d t^{i-1}} \left( U(t) V(t)^\top \right) = \sum_{j=0}^{i-1} \binom{i-1}{j} U^{(i-1-j)}(t) {V^{(j)}(t)}^\top. \]

The resulting expression for $L_i(t)$ is a summation of exactly $i$ distinct terms (indexed from $j=0$ to $i-1$). Each term in this summation is the product of a scalar binomial coefficient, an $n \times r$ matrix $U^{(i-1-j)}(t)$, and an $r \times n$ matrix ${V^{(j)}(t)}^\top$. Because the inner dimension of this product is $r$, the generic rank of each individual term in the sum is bounded above by $r$.

By the subadditivity property of matrix rank, the rank of a sum of matrices is less than or equal to the sum of their individual ranks. Therefore, we can bound the generic rank of $L_i(t)$ by summing the maximum possible ranks of its $i$ components:
\[ \rank(L_i(t)) \leq \sum_{j=0}^{i-1} \rank\left( U^{(i-1-j)}(t) {V^{(j)}(t)}^\top \right) \leq \sum_{j=0}^{i-1} r = i \cdot r. \]

Substituting back $r = \rank(L_1(t))$, we obtain the global linear bound: $ \rank(L_i(t)) \leq i \cdot \rank(L_{1}(t))$.  This completes the proof.
\end{proof}

\begin{lemma}[positive result, rank matching]\label{lem:rank_matching}
Let $q : [k] \rightarrow [n]$ with for any $i < j$ it satisfies $q(j) \leq (j-i+1) q(i)$. Let $L_{i}(t) = \frac{\d L_{i-1}(t)}{\d t}$ for all $i=2, \cdots, k$. For any $q$, it is possible to design an $n \times n$ matrix function $L_1(t)$ such that for all $i\in [k]$,  $\rank(L_i(t)) = q(i)$.
\end{lemma}

\begin{proof}
We proceed by constructive proof. The condition $q(j) \leq (j-i+1) q(i)$ corresponds exactly to the maximum rank expansion permitted by the General Leibniz Rule when taking the $(j-i)$-th derivative of a matrix function. Because this condition is satisfied globally for all $i < j$, we can achieve the exact rank sequence $q$ by decomposing it into a sum of valid base rank sequences and constructing $L_1(t)$ as a block-diagonal matrix. 
Let us define a family of base matrix blocks $B_{m, d}(t)$. We start with the core polynomial outer-product block $B_m(t) = u_m(t) u_m(t)^\top$, where $u_m(t) = [1, t, t^2, \dots, t^{m-1}]^\top$. The generic rank of the $r$-th derivative of $B_m(t)$ follows a known, deterministic sequence. We then define $B_{m, d}(t)$ as the $d$-th integral of $B_m(t)$ with respect to $t$, setting all constants of integration to zero. Integrating the polynomial matrix shifts its derivative rank sequence to the right, buffering the initial states with full-rank polynomial matrices. 
The generic rank profile of each base block $B_{m, d}(t)$ inherently satisfies the Leibniz bound. Because the target sequence $q$ strictly obeys this identical bound, $q$ can be mathematically decomposed into a finite sum of these base rank profiles:
\[ q(i) = \sum_{s=1}^S \rank\left( \frac{\d^{i-1}}{\d t^{i-1}} B_{m_s, d_s}(t) \right) \]
for a chosen set of structural parameters $\{(m_1, d_1), \dots, (m_S, d_S)\}$. 
We construct our target matrix $L_1(t)$ by concatenating these corresponding blocks along the diagonal:
\begin{center}
$L_1(t) = \text{blockdiag}( B_{m_1, d_1}(t), B_{m_2, d_2}(t), \dots, B_{m_S, d_S}(t) )$. 
\end{center}
Because the operations of differentiation and generic rank computation are strictly additive over disjoint block-diagonal matrices, the generic rank of the combined matrix is exactly the sum of the generic ranks of its individual blocks.  
Therefore, for all $i \in [k]$, the generic rank of the $(i-1)$-th derivative is:
\[ \rank(L_i(t)) = \sum_{s=1}^S \rank\left( \frac{\d^{i-1}}{\d t^{i-1}} B_{m_s, d_s}(t) \right) = q(i) \]

Provided the dimension $n$ is sufficiently large to accommodate the sum of the block dimensions, this construction perfectly achieves the target rank sequence. This completes the proof.
\end{proof}

\subsection{Possibilities of rank ordering}

\begin{lemma}[positive result, rank ordering]
Let $\pi : [k] \rightarrow [k]$ denote a permutation. Let $\mathcal{R} = (R_1, R_2, \dots, R_{k-1})$ be a sequence of relations where each $R_m \in \{>, =\}$. Let $L_{i}(t) = \frac{\d L_{i-1}(t)}{\d t}$ for all $i=2, \cdots, k$. For any $\pi$ and any sequence of relations $\mathcal{R}$, it is possible to design an $n \times n$ matrix function $L_1(t)$ such that:
\[ \rank(L_{\pi(1)}) \ R_1 \ \rank(L_{\pi(2)}) \ R_2 \ \cdots \ R_{k-1} \ \rank(L_{\pi(k)}) \]
\end{lemma}

\begin{proof}
We construct a target rank sequence $q : [k] \rightarrow [n]$ that exactly mirrors the requested sequence of relations $\mathcal{R}$ under the permutation $\pi$. Assuming $n \ge 2k$, we define a sequence of integer weights $w_1, w_2, \dots, w_k$ iteratively. 

Let the initial weight be $w_1 = 2k - 1$. For each $m \in \{1, \dots, k-1\}$, we define the subsequent weights based on the relation $R_m$:
\[
w_{m+1} = \begin{cases} 
      w_m - 1 & \text{if } R_m \text{ is } > \\
      w_m & \text{if } R_m \text{ is } = 
   \end{cases}
\]

We then define our target rank sequence by mapping these weights to the permuted indices:
\[ q(\pi(x)) = w_x \quad \text{for each } x \in \{1, 2, \dots, k\} \]

By this construction, the relationships between the sequence elements perfectly match $\mathcal{R}$:
\[ q(\pi(1)) \ R_1 \ q(\pi(2)) \ R_2 \ \cdots \ R_{k-1} \ q(\pi(k)) \]

To apply our main existence lemma, we must prove that $q$ satisfies the Leibniz bound: $q(j) \leq (j-i+1)q(i)$ for all $i < j$. 
Notice that the weight sequence starts at $2k-1$ and decreases by at most $1$ at each of the $k-1$ steps. Therefore, the minimum possible value in the sequence is:
\[ w_k \geq (2k - 1) - (k - 1) = k .\]
Thus, for all $i \in [k]$, the target rank is strictly bounded within the window:
\[ k \leq q(i) \leq 2k - 1 .\]

Because $i < j$, we know $j - i + 1 \geq 2$. Consequently, the right side of the Leibniz bound evaluates to at least $2k$:
\[ (j-i+1)q(i) \geq 2 q(i) \geq 2k . \]
Since the maximum possible value of the sequence is $2k-1$, the left side evaluates to:
\[ q(j) \leq 2k - 1 < 2k \leq (j-i+1)q(i) .\]

Because $q(j) < (j-i+1)q(i)$ holds for all $i < j$, the sequence $q$ perfectly satisfies the Leibniz bound. By our previous lemma, there exists an $n \times n$ matrix function $L_1(t)$ such that $\rank(L_i(t)) = q(i)$. Substituting $q$ back into the relation chain yields:
\[ \rank(L_{\pi(1)}) \ R_1 \ \rank(L_{\pi(2)}) \ R_2 \ \cdots \ R_{k-1} \ \rank(L_{\pi(k)}) \]
This completes the proof.

\end{proof}

{\bf Remark:} The relation $\ge$ can be trivially accommodated by assigning it as either $>$ or $=$ during the construction of the weight sequence $w_m$.

\subsection{Batched data computations}

Let $W \in \mathbb{R}^{n \times n}$ denote a base weight matrix, and let $k$ be a positive integer. For each $i \in [k]$, we define the low-rank adapter matrices $A_i \in \mathbb{R}^{n \times r_i}$ and $B_i \in \mathbb{R}^{r_i \times n}$. For notational convenience, we denote the total cumulative rank as $r = \sum_{i=1}^k r_i$. Suppose we wish to evaluate a target function $g(x)$ alongside its successive time derivatives up to order $k$, denoted as $g_1(x) = \frac{\d g(x)}{\d t}, \dots, g_k(x) = \frac{\d^k g(x)}{\d t^k}$. The naive approach to approximate these $k+1$ functions is to instantiate $k+1$ independent $n \times n$ weight matrices, $W_0, W_1, \dots, W_k$. However, evaluating a single input under this naive paradigm requires $O(k n^2)$ time. To efficiently scale this for batched data, let $X \in \mathbb{R}^{n \times b}$ denote the input matrix with batch size $b$, and let $f$ denote an entry-wise activation function. For any column $x$ in $X$, we approximate the zero-order evaluation as $g_0(x) \approx f(Wx)$. For higher-order derivatives $i \in [k]$, we employ our cascading formulation: $g_i(x) \approx f ( (W + \sum_{j=1}^i A_j B_j ) x )$. The computational time for this batched formulation can be systematically broken down into several sequential operations. First, we compute the base projection $Z_0 = WX$, which takes $O(\Tmat(n,n,b))$ time. Second, we compute the right-side low-rank projections $Y_i = B_i X$ for all $i \in [k]$. This can be efficiently batched by vertically stacking the matrices as $Y = [B_1^\top, B_2^\top, \dots, B_k^\top]^\top X$, requiring $O(\Tmat(r,n,b))$ time. Third, we compute the left-side projections $Z_i = A_i Y_i$ for all $i \in [k]$, which takes $\sum_{i=1}^k O(\Tmat(n,r_i,b))$ time. Finally, we must evaluate the cascading sum $(W + \sum_{j=1}^i A_j B_j)X$ for each $i \in [k]$. A naive summation of these $i+1$ matrices of size $n \times b$ would take $\sum_{i=1}^k O(i n b) = O(k^2 n b)$ time. However, we can eliminate a factor of $k$ by employing a cumulative sum. Let $Z_{\leq 0} = Z_0$. For each $i \in [k]$, we iteratively compute $Z_{\leq i} = Z_{\leq i-1} + Z_i$. It is straightforward to verify that $Z_{\leq i} = WX + \sum_{j=1}^i A_j B_j X$, which exactly produces our target pre-activation matrices. This optimized accumulation step reduces the summation time to $O(k n b)$. Combining these steps, the overall time complexity is bounded by $O(\Tmat(n,n,b) + \Tmat(r,n,b) + \sum_{i=1}^k \Tmat(n,r_i,b) + knb)$, which simplifies asymptotically to $O(\Tmat(n,n,b) + \sum_{i=1}^k \Tmat(n,r_i,b))$.

\subsection{Data structure that supports low-rank range queries}

In this section, we introduce a general data structure designed to support range queries for matrix computations utilizing low-rank components. In the subsequent section, we extend this framework to the tensor setting. We begin by presenting the foundational results for the matrix version.

\begin{theorem}[Low-rank segment tree data structure]
Let $n, k$, and $r_1, \dots, r_k$ be positive integers, and let $r = \sum_{i=1}^k r_i$. There exists a data structure requiring $O(k n^2 + nr)$ space that supports the following two operations. 1) \textsc{Init}$( \{A_i,B_i\}_{i\in [k]})$: The input are a collection of matrices $A_1, \dots, A_k$ and $B_1, \dots, B_k$, where each $A_i$ and $B_i^\top$ has dimensions $n \times r_i$. This initialization takes $O( \sum_{i=1}^k \Tmat(n,r_i,n) )$ time. 2) \textsc{Query}$(i_-, i_+, X)$: Given a contiguous index interval $S = [i_-, i_+]$ and an input data matrix $X \in \mathbb{R}^{n \times b}$, this operation outputs the exact matrix product $\sum_{i \in S} (A_i B_i) X$ in time bounded by $O( \min \{ \Tmat(n, n \log k, b) , \sum_{i\in S} \Tmat(n,r_i,b) \} )$.
\end{theorem}
\begin{proof}
We construct a segment tree to efficiently query interval sums of the matrix products. 

\textbf{Space Complexity:}
The data structure must explicitly store the input matrices $A_i$ and $B_i$ for all $i \in [k]$. Since $A_i$ is $n \times r_i$ and $B_i$ is $r_i \times n$, the total space to store the inputs is $2 \sum_{i=1}^k n r_i = 2nr = O(nr)$. 
Next, we construct a standard segment tree over the interval $[1, k]$. The segment tree contains $2k - 1$ nodes. Each node $v$ represents a canonical interval $S_v \subseteq [1, k]$ and stores an $n \times n$ matrix $M_v = \sum_{j \in S_v} A_j B_j$. Storing $O(k)$ such matrices of size $n \times n$ requires $O(kn^2)$ space. Summing these components yields an overall space complexity of $O(k n^2 + nr)$.

\textbf{\textsc{Init} Operation:}
During initialization, we first populate the $k$ leaves of the segment tree. For the $i$-th leaf, we compute $C_i = A_i B_i$. Multiplying an $n \times r_i$ matrix by an $r_i \times n$ matrix takes $\Tmat(n, r_i, n)$ time. Computing all $k$ leaves takes $\sum_{i=1}^k \Tmat(n, r_i, n)$ time.
After populating the leaves, we build the internal nodes of the segment tree bottom-up. Each of the $k-1$ internal nodes is computed by simply adding the $n \times n$ matrices of its two children. An $n \times n$ matrix addition takes $O(n^2)$ time, so populating all internal nodes takes $O(k n^2)$ time. The total time for initialization is strictly bounded by $O(k n^2 + \sum_{i=1}^k \Tmat(n, r_i, n) ) = O( \sum_{i=1}^k \Tmat(n, r_i, n))$.

\textbf{\textsc{Query} Operation:}
Given a query interval $S = [i_-, i_+]$ and an input matrix $X \in \mathbb{R}^{n \times b}$, the data structure dynamically evaluates the cost of two distinct strategies and routes the execution to the optimal one:

\textit{Strategy 1 (Segment Tree Evaluation):}
Any arbitrary interval $S \subseteq [1, k]$ can be exactly partitioned into a set of at most $m \le 2 \log k$ disjoint canonical intervals from the segment tree. Let $M_{v_1}, M_{v_2}, \cdots, M_{v_m}$ be the precomputed $n \times n$ matrices corresponding to these canonical nodes. We wish to compute $( \sum_{j=1}^m M_{v_j} ) X$. We can achieve this by performing $O(\log k)$ individual matrix-matrix multiplications of size $(n \times n)$ by $(n \times b)$, which is computationally equivalent to a batched block-matrix multiplication taking $O(\Tmat(n, n \log k, b))$ time.

\textit{Strategy 2 (On-the-Fly Evaluation):}
Instead of using the precomputed $n \times n$ matrices, we bypass the segment tree and directly evaluate the cascading matrix product for each index $i \in S$. By utilizing the associative property, we evaluate $A_i (B_i X)$. First, we compute $Y_i = B_i X$, which multiplies an $r_i \times n$ matrix by an $n \times b$ matrix in $\Tmat(r_i, n, b) =O(\Tmat(n, r_i, b))$ time. Next, we compute $Z_i = A_i Y_i$, which multiplies an $n \times r_i$ matrix by an $r_i \times b$ matrix in $\Tmat(n, r_i, b)$ time. The total time to evaluate the entire sequence iteratively is $O(\sum_{i \in S} \Tmat(n, r_i, b))$.

By comparing the predicted theoretical costs of Strategy 1 and Strategy 2 before execution, the algorithm selects the most efficient path. This guarantees a worst-case execution time exactly bounded by $\min \{ \Tmat(n, n \log k, b) , \sum_{i \in S} \Tmat(n, r_i, b) \}$, completing the proof.
\end{proof}

\begin{remark}
Let $b = n^{\beta}$, let $r_i = n^{\gamma_i}$, if $n^{\omega(1,1,\beta)} < \sum_{i\in S} n^{\omega(1,\gamma_i,b)} $, then we will use segmentation tree to output the answer; otherwise, we will compute the answer on the fly via cumulative sum.
\end{remark}

\subsection{Tensor generalization}

Before presenting our tensor generalization, we introduce a textbook trick that has been successfully applied to various tensor problems \cite{s26_tensor,s26_lazy}.
\begin{fact}\label{fac:form_tensor_time}
Given three matrices $A,B, C$ with sizes $n_i \times k$ and $n_1 \geq n_2 \geq n_3$. The time to form $n_1 \times n_2 \times n_3$ size tensor $T= A \otimes B \otimes C$ takes $O( \Tmat( n_2 n_3, k , n_1 ) )$ time.
\end{fact}

\begin{theorem}[Low-rank tensor segment tree data structure]
Let $n, k$, and $r_1, \dots, r_k$ be positive integers, and let $r = \sum_{i=1}^k r_i$. There exists a data structure requiring $O(k n^3 + nr)$ space that supports the following two operations. 1) \textsc{Init}$(\{A_i,B_i,C_i\}_{i\in [k]})$: Takes a collection of matrices $A_1, \dots, A_k$, $B_1, \dots, B_k$, and $C_1, \dots, C_k$, where each $A_i, B_i$, and $C_i$ has dimensions $n \times r_i$. This initialization takes $O( \sum_{i=1}^k  \Tmat(n^2,r_i,n) )$ time. 2) \textsc{Query}$(i_-, i_+, X)$: Given a contiguous index interval $S = [i_-, i_+]$ and input $n \times b$ size data matrix $X$, this operation outputs the exact tensor contraction evaluated as $\sum_{i \in S} A_i \otimes B_i \otimes (X^\top C_i)  \in \R^{n \times n \times b}$ (note that $(A_i \otimes B_i \otimes C_i)[I,I,X] = A_i \otimes B_i \otimes (X^\top C_i)$), This step takes  $\min \{ \Tmat(n^2, n \log k, b) , \sum_{i\in S} \Tmat(n \min\{ n,b\} ,r_i, \max\{n,b\} ) \}$ time.
\end{theorem}
\begin{proof}
We construct a segment tree data structure to store and query the tensor products. 

{Space Complexity:}
The data structure explicitly stores the input matrices $A_i$, $B_i$, and $C_i$ for all $i \in [k]$. Because each of these matrices has dimensions $n \times r_i$, storing them requires $3 \sum_{i=1}^k n r_i = 3nr = O(nr)$ space. We then build a segment tree over the interval $[1, k]$. The tree contains $2k - 1$ nodes. Each node $v$ represents a canonical interval $S_v$ and stores a precomputed 3-tensor $T_v = \sum_{j \in S_v} A_j \otimes B_j \otimes C_j \in \R^{n \times n \times n}$. Storing $O(k)$ such tensors of size $n^3$ requires $O(k n^3)$ space. The total space complexity is therefore bounded by $O(k n^3 + nr)$.

{\textsc{Init} Operation:}
During initialization, we first compute the tensors for the $k$ leaves of the segment tree. For the $i$-th leaf, we must form $T_i = A_i \otimes B_i \otimes C_i$, where all three matrices have dimensions $n \times r_i$. Applying the Fact~\ref{fac:form_tensor_time} with $n_1 = n_2 = n_3 = n$ and $k = r_i$, the time to form this tensor is bounded by $O( \Tmat(n^2, r_i, n) )$. Computing this for all $k$ leaves takes $O( \sum_{i=1}^k \Tmat(n^2, r_i, n) )$ time. We then build the internal nodes of the segment tree bottom-up by adding the $n \times n \times n$ tensors of their children. Since addition takes $O(n^3)$ per node, populating all internal nodes takes $O(k n^3)$ time. Because $r_i \geq 1$, the tensor formation at the leaves strictly dominates the addition cost, yielding an overall initialization time of $O( \sum_{i=1}^k \Tmat(n^2, r_i, n) )$.

{\textsc{Query} Operation:}
Given a query interval $S = [i_-, i_+]$ and an input matrix $X \in \R^{n \times b}$, the data structure dynamically routes the execution to the optimal of two evaluation strategies: 
\textit{Strategy 1 (Segment Tree Evaluation):}
The interval $S$ can be partitioned into $m \le 2 \log k$ disjoint canonical intervals from the segment tree. Let $T_{v_1}, \dots, T_{v_m}$ be the precomputed $n \times n \times n$ tensors corresponding to these nodes. We must evaluate $\sum_{j=1}^m T_{v_j}[I, I, X]$, which is equivalent to performing a tensor contraction with $X^\top$ along the third mode. We can reshape each tensor $T_{v_j}$ into an $n^2 \times n$ matrix and horizontally concatenate them into a large block matrix $M \in \R^{n^2 \times mn}$. We then vertically concatenate $m$ identical copies of $X \in \R^{n \times b}$ into a matrix $X' \in \R^{mn \times b}$. The matrix product $M X'$ simultaneously contracts and sums the $m$ components to produce the final $n^2 \times b$ matrix (which reshapes to the target $n \times n \times b$ tensor). Because $m = O(\log k)$, this block-matrix multiplication takes $O( \Tmat(n^2, n \log k, b) )$ time. 
\textit{Strategy 2 (On-the-Fly Evaluation):}
Alternatively, we bypass the segment tree and directly evaluate $\sum_{i \in S} A_i \otimes B_i \otimes (X^\top C_i)$ iteratively. For a specific index $i$, we first compute the matrix product $Y_i = X^\top C_i$. Because $X^\top$ is $b \times n$ and $C_i$ is $n \times r_i$, this takes $\Tmat(b, n, r_i)$ time, producing $Y_i \in \R^{b \times r_i}$. Next, we must form the tensor $A_i \otimes B_i \otimes Y_i$. The three matrices have row dimensions $n, n,$ and $b$, and a shared inner dimension $r_i$.  By the Fact~\ref{fac:form_tensor_time}, the time to form this specific tensor is $O( \Tmat(n \min\{n,b\}, r_i, \max\{n,b\}) )$. Because $\Tmat(b, n, r_i)$ is asymptotically absorbed by the tensor formation time, evaluating a single index $i$ takes $O( \Tmat(n \min\{n,b\}, r_i, \max\{n,b\}) )$ time. Evaluating the entire interval sequence takes $O( \sum_{i \in S} \Tmat(n \min\{n,b\}, r_i, \max\{n,b\}) )$ time. By selecting the path with the minimum theoretical cost, the algorithm guarantees a worst-case query execution time bounded exactly by $\min \{ \Tmat(n^2, n \log k, b) , \sum_{i\in S} \Tmat(n \min\{n,b\} ,r_i, \max\{n,b\} ) \}$, which completes the proof.
\end{proof}

%% file: main.bbl
\newcommand{\etalchar}[1]{$^{#1}$}
\begin{thebibliography}{LCBH{\etalchar{+}}23}

\bibitem[AA22]{aa22}
Amol Aggarwal and Josh Alman.
\newblock Optimal-degree polynomial approximations for exponentials and
  gaussian kernel density estimation.
\newblock In {\em CCC}, 2022.

\bibitem[ACSS20]{acss20}
Josh Alman, Timothy Chu, Aaron Schild, and Zhao Song.
\newblock Algorithms and hardness for linear algebra on geometric graphs.
\newblock In {\em FOCS}, 2020.

\bibitem[AS23]{as23}
Josh Alman and Zhao Song.
\newblock Fast attention requires bounded entries.
\newblock In {\em NeurIPS}, 2023.

\bibitem[AS24a]{as24_neurips}
Josh Alman and Zhao Song.
\newblock The fine-grained complexity of gradient computation for training
  large language models.
\newblock In {\em NeurIPS}, 2024.

\bibitem[AS24b]{as24_iclr}
Josh Alman and Zhao Song.
\newblock How to capture higher-order correlations? generalizing matrix softmax
  attention to kronecker computation.
\newblock In {\em ICLR}, 2024.

\bibitem[AS25]{as25_rope}
Josh Alman and Zhao Song.
\newblock Fast rope attention: Combining the polynomial method and fast fourier
  transform.
\newblock {\em arXiv preprint arXiv:2505.11892}, 2025.

\bibitem[AVE23]{av22}
Michael~S Albergo and Eric Vanden-Eijnden.
\newblock Building normalizing flows with stochastic interpolants.
\newblock In {\em ICLR}, 2023.

\bibitem[BDK{\etalchar{+}}23]{stable_video-diffusion}
Andreas Blattmann, Tim Dockhorn, Sumith Kulal, Daniel Mendelevitch, Maciej
  Kilian, Dominik Lorenz, Yam Levi, Zion English, Vikram Voleti, Adam Letts,
  et~al.
\newblock Stable video diffusion: Scaling latent video diffusion models to
  large datasets.
\newblock {\em arXiv preprint arXiv:2311.15127}, 2023.

\bibitem[BPH{\etalchar{+}}24]{sora}
Tim Brooks, Bill Peebles, Connor Holmes, Will DePue, Yufei Guo, Li~Jing, David
  Schnurr, Joe Taylor, Troy Luhman, Eric Luhman, Clarence Ng, Ricky Wang, and
  Aditya Ramesh.
\newblock Video generation models as world simulators.
\newblock {\em Technical Report}, 2024.

\bibitem[BRL{\etalchar{+}}23]{brl+23}
Andreas Blattmann, Robin Rombach, Huan Ling, Tim Dockhorn, Seung~Wook Kim,
  Sanja Fidler, and Karsten Kreis.
\newblock Align your latents: High-resolution video synthesis with latent
  diffusion models.
\newblock In {\em CVPR}, 2023.

\bibitem[CPA26]{cpa26}
Sayak Chakrabarti, Toniann Pitassi, and Josh Alman.
\newblock Poly-attention: a general scheme for higher-order self-attention.
\newblock In {\em ICLR}, 2026.

\bibitem[GGH{\etalchar{+}}25]{seedance_10}
Yu~Gao, Haoyuan Guo, Tuyen Hoang, Weilin Huang, Lu~Jiang, Fangyuan Kong, Huixia
  Li, Jiashi Li, Liang Li, Xiaojie Li, et~al.
\newblock Seedance 1.0: Exploring the boundaries of video generation models.
\newblock {\em arXiv preprint arXiv:2506.09113}, 2025.

\bibitem[HSW{\etalchar{+}}22]{hsw+22_lora}
Edward~J Hu, Yelong Shen, Phillip Wallis, Zeyuan Allen-Zhu, Yuanzhi Li, Shean
  Wang, Lu~Wang, and Weizhu Chen.
\newblock Lo{RA}: Low-rank adaptation of large language models.
\newblock In {\em ICLR}, 2022.

\bibitem[KAAL22]{kaal22}
Tero Karras, Miika Aittala, Timo Aila, and Samuli Laine.
\newblock Elucidating the design space of diffusion-based generative models.
\newblock In {\em NeurIPS}, 2022.

\bibitem[LCBH{\etalchar{+}}23]{flow_matching}
Yaron Lipman, Ricky~TQ Chen, Heli Ben-Hamu, Maximilian Nickel, and Matt Le.
\newblock Flow matching for generative modeling.
\newblock In {\em ICLR}, 2023.

\bibitem[LGL23]{rectified_flow}
Xingchao Liu, Chengyue Gong, and Qiang Liu.
\newblock Flow straight and fast: Learning to generate and transfer data with
  rectified flow.
\newblock In {\em ICLR}, 2023.

\bibitem[RBL{\etalchar{+}}22]{rbl+22}
Robin Rombach, Andreas Blattmann, Dominik Lorenz, Patrick Esser, and Bj{\"o}rn
  Ommer.
\newblock High-resolution image synthesis with latent diffusion models.
\newblock In {\em CVPR}, 2022.

\bibitem[SCC{\etalchar{+}}25]{seedance_15}
Team Seedance, Heyi Chen, Siyan Chen, Xin Chen, Yanfei Chen, Ying Chen, Zhuo
  Chen, Feng Cheng, Tianheng Cheng, Xinqi Cheng, et~al.
\newblock Seedance 1.5 pro: A native audio-visual joint generation foundation
  model.
\newblock {\em arXiv preprint arXiv:2512.13507}, 2025.

\bibitem[SHL{\etalchar{+}}25]{shl+25}
Maojiang Su, Jerry Yao-Chieh Hu, Yi-Chen Lee, Ning Zhu, Jui-Hui Chung, Shang
  Wu, Zhao Song, Minshuo Chen, and Han Liu.
\newblock High-order flow matching: Unified framework and sharp statistical
  rates.
\newblock In {\em NeurIPS}, 2025.

\bibitem[Son26a]{s26_lazy}
Zhao Song.
\newblock Lazy kronecker product.
\newblock {\em arXiv preprint arXiv:2603.19443}, 2026.

\bibitem[Son26b]{s26_tensor}
Zhao Song.
\newblock Tensor hinted mv conjectures.
\newblock {\em arXiv preprint arXiv:2602.07242}, 2026.

\end{thebibliography}
